%% file: acl_latex.tex
\pdfoutput=1
\documentclass[11pt]{article}

\usepackage[final]{acl}

\usepackage{times}
\usepackage{latexsym}

\usepackage[T1]{fontenc}
\usepackage[utf8]{inputenc}

\usepackage{microtype}

\usepackage{inconsolata}

\usepackage{graphicx}
\usepackage{amsmath, amsthm}
\newtheorem{proposition}{Proposition}
\usepackage{booktabs}
\usepackage{algorithm}
\usepackage{algpseudocode}
\usepackage{array}
\usepackage{longtable}
\usepackage{graphicx}
\usepackage{subcaption}
\usepackage{enumitem}
\usepackage{multirow}
\usepackage{amsfonts}

\title{DavIR: Data Selection via Implicit Reward for Large Language Models}
\author{ \bf
Haotian Zhou$^{\dagger, 1}$,
Tingkai Liu$^{\dagger,2}$, 
Qianli Ma$^{1}$,
Yufeng Zhang$^{1}$,
Jianbo Yuan$^{1}$\\
\bf 
Pengfei Liu$^{3,}$\thanks{Corresponding authors},
Yang You$^{4,\ast}$,
Hongxia Yang$^{1,\ast}$ \\
$^\dagger$ Equal Contribution \\
$^1$ByteDance, Inc. \\
$^2$NeuroAI Scholar, Cold Spring Harbor Laboratory \\
$^3$Generative Artificial Intelligence Research Lab, Shanghai Jiao Tong University \\
$^4$School of Computing, National University of Singapore \\
}

\begin{document}
\maketitle
\begin{abstract}
We introduce DavIR, a model-based data selection method for post-training Large Language Models. DavIR generalizes Reducible Holdout Loss to core-set selection problem of causal language modeling, and quantifies the ``learnability'' of a given datum with respect to a pre-trained LLM based on relative reduction in loss during fine-tuning, a metric we show to be closely related to the implicit reward model described in Direct Preference Optimization (DPO). 
We show that 6\% of Alpaca dataset selected with DavIR can steer both the LLaMA and Gemma model family to produce superior performance compared to the same models trained on the full 52K dataset. We also show that Alpaca dataset compressed with DavIR can be combined with GSM8K dataset to effectively balance open-domain freeform QA and mathematical reasoning capabilities. Finally, we apply the DavIR objective to DPO and develop a normalized DavIR-DPO objective which improves alignment performance of Zephyr-7B-SFT model by 8\% (relative) on AlpacaEval, compared against training on vanilla DPO objective.
\end{abstract}

\section{Introduction}
\input{main/introduction}

\section{Background and Related Works}
\input{main/related_work}

\section{DavIR: Data Selection via Implicit Reward}
\input{main/method}

\section{Experiments and Results}
\input{main/exp}

\section{Conclusion}
\input{main/conclusion}

\section{Limitations \& Discussions}
\input{main/discussion}

\section{Ethics Statement}
\input{main/ethics}

\bibliography{references}

\clearpage
\newpage
\appendix
\input{main/appendix}

\end{document}

%% file: main/introduction.tex
Large Language Models (LLMs)~\citep{brown2020language, chowdhery2022palm, touvron2023llama, ouyang2022training} have sparked a revolution in the field of Natural Language Processing (NLP), with far reaching impacts in domains such as law~\citep{cui2023chatlaw}, medical~\citep{singhal2022large} and finance~\citep{wu2023bloomberggpt}.

A critical step in the current paradigm of post-training LLMs is Supervised/Instruction Fine-tuning (SFT/IFT), which enables pre-trained models to exhibit strong instruction-following capabilities~\citep{chung2022scaling, ouyang2022training, touvron2023llama, wang2022self, zheng2023judging}. Selecting the most effective training data during this stage is particularly important since effective steering of LLM during SFT could be achieved by just a few thousand carefully curated data~\citep{zhou2023lima}. Previous approaches to selecting SFT training data focused on data quality and diversity~\citep{ji2023exploring, zhou2023lima, chen2023alpagasus, chen2023maybe, li2023selfalignment}, guided by the intuition of encouraging LLMs to output accurate and reliable information while maintaining generalization capabilities to a wide range of tasks and scenarios.

However, by focusing on the quality and diversity of the data, existing methods are \emph{data-centric}, and are agnostic to the capabilities of the pre-trained model upon which fine-tuning occurs. Instead, following the ``Superficial Alignment Hypothesis'' ~\citep{zhou2023lima} which postulates that fine-tuning process unlocks the capabilities of pre-trained LLMs, we seek a model-centric data selection algorithm that chooses data that:
\begin{enumerate}[leftmargin=*,noitemsep,topsep=5pt]
    \item Quantifies the degree to which a model ``learns'' a data before and after training;
    \item Does not require querying closed-source teacher models which may lead to security concerns;
    \item Is theoretically grounded in the implicit reward function of the underlying LLM (see Section.~\ref{sect:method}).
\end{enumerate}
We note that the previously proposed Reducible Holdout Loss~\citep{rho} admits a simplification that satisfy the three requirements above~\citep{rafailov2023direct}. However, when applying RHO-like objectives to language modeling tasks, we observed a significant challenge: the RHO metric is highly correlated with the sequence length of the input data. This correlation introduces an undesirable bias in the data selection process, reducing the core-set selection to an approximation of length-based filtering. We show that this issue is inherent to the \emph{sequential} nature of language modeling in state-of-the-art LLMs, and cannot be resolved by normalizing the total cross entropy loss by number of tokens in a datum (document).

Instead, a subtle yet crucial change in normalization of the RHO objective - normalizing with reference model loss instead of number of tokens - dramatically reduced the length dependency of the object. We term this modified RHO objective, and the consequent data selection method, \textbf{DavIR} ({\bf Da}ta Selection {\bf v}ia {\bf I}mplicit {\bf R}eward).

We demonstrate the effectiveness of DavIR across model families (LLaMA~\citep{touvron2023llama}, Gemma ~\citep{gemma,gemma2}) and across data domains (Alpaca~\citep{alpaca},  LIMA~\citep{zhou2023lima}, GSM8K~\citep{cobbe2021training}) and across benchmarks benchmarks (Self-Instruct~\citep{wang2022self}, Vicuna~\citep{zheng2023judging}, Koala~\citep{koala_blogpost_2023}, OpenAssistant~\citep{köpf2023openassistant}, Helpful Base~\citep{bai2022training}, GSM8K~\citep{cobbe2021training}). We show that DavIR outperforms \emph{all} (to the best of author's knowledge) state-of-the-art core-set selection methods across benchmarks.
% {\color{add all baselines, new summary figure required to highlight performance gains.}}

% baseline data selection methods: 1) fine-tuning on full training data, 2) random data selection and 3) selection using teacher LLM \citep{chen2023alpagasus}. As shown in Figure.~\ref{fig:fig1}, across benchmarks with both human and GPT-4 as referee for evaluation~\citep{zheng2023judging}, DavIR shows superior performance to all baseline methods with a fraction of the training data (6.15\% of Alpaca). Finally, we showed that mixture of math and open-domain QA data selected via DavIR result boosts both mathematical and general conversation capabilities using 16.7\% of the full dataset.

Finally, as the introduction of normalization in the DavIR objective led to a deviation from the implicit reward model given by the vanilla DPO objective, we propose {\bf DavIR-DPO} that incorporates the normalization proposed in the current work. 
We show that DavIR-DPO metric is the least correlated with the difference in length of paired responses in UltraFeedback dataset \citep{ultrafeedback}, and training Zephyr-7B-SFT model using the DavIR-DPO metric led to an 8\% boost of length-controlled performance on AlpacaEval~\citep{alpaca_eval, alpaca_eval_v2} as compared to when trained using the vanilla DPO objective.

\begin{figure}[!t] 
\centering 
\includegraphics[width=\linewidth]{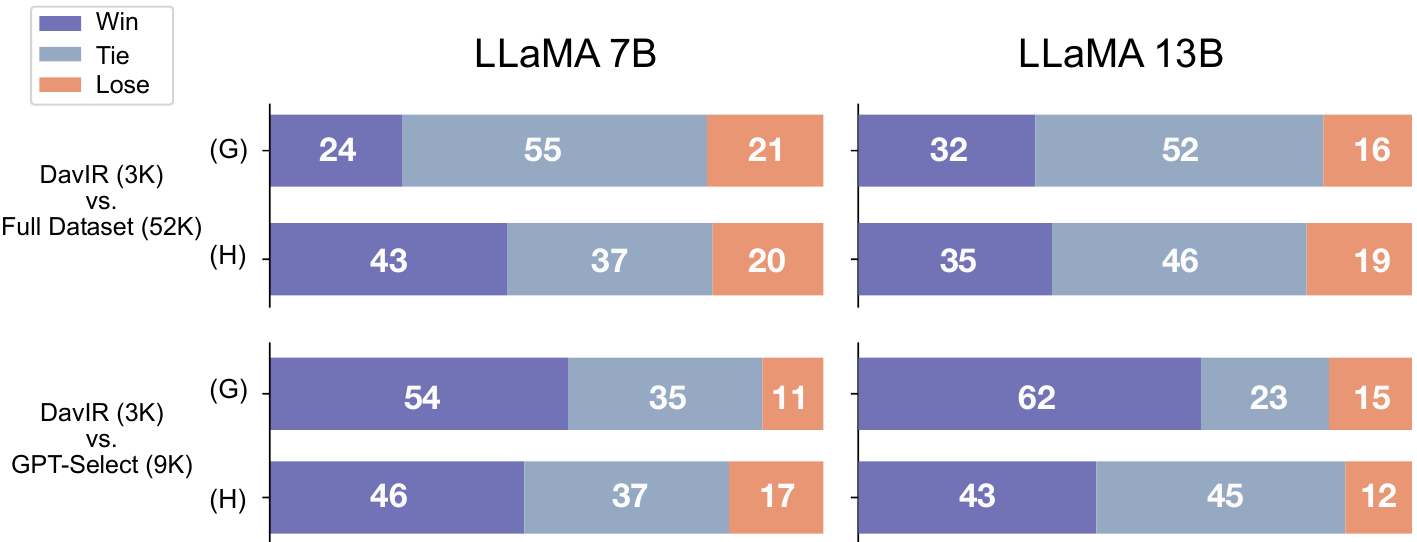} 
\caption{
\textbf{DavIR outperforms full data fine-tuning and data selection based on teacher LLM across model scales.} 
Performance comparison of 7B and 13B parameter models fine-tuned with data selected using DavIR (3,000 items), the full Alpaca dataset (52K), and data filtered using ChatGPT (9,229 items). ``G" represents evaluation using GPT-4, and ``H" represents human evaluation. The statistical significance of performance gain of DavIR over training on full dataset and other core-set selection methods are established in subsequent sections.}
\label{fig:fig1}
\end{figure}

%% file: main/related_work.tex
\paragraph{Supervised (Instruction) Fine-tuning of LLM}
In training LLMs, Supervised (Instruction) Fine-tuning (SFT/IFT) plays a pivotal role during post-training by fine-tune LLMs with a small amount of data to enable instruction-following and multi-round dialogue. Two predominant methods of collecting SFT training data are 1) distillation from teacher models (e.g. Self-Instruct~\citep{wang2022self}, Alpaca ~\citep{alpaca}, Evol-Instruct~\citep{xu2023wizardlm}) and 2) manual annotation (e.g. InstructGPT~\citep{ouyang2022training}, Vicuna~\citep{zheng2023judging}, LIMA ~\citep{zhou2023lima}). 

\paragraph{Implicit Reward in Direct Preference Optimization}
First proposed in~\citep{rafailov2023direct}, Direct Preference Optimization (DPO) emerged as a post-training method following supervised-fintuning. DPO simplifies the RLHF pipeline by directly optimizing a language model using preference data, eliminating the need for explicit reward modeling and reinforcement learning. The simplicity of DPO and its effectiveness has led to wide adoption across models. A large body of follow-up works have been proposed that modify the DPO objective to improve robustness \citep{dpo_ipo, dpo_exo, dpo_robust, dpo_nca, dpo_sppo}, address issues of data scarcity \citep{dpo_rso, dpo_bco} or provide stronger control over likelihood of producing winning and losing responses \citep{dpo_apo, dpo_aot}. Recent works have also began to explore length-dependencies of the DPO objective \citep{multi_dpo, park2024disentangling}.

\paragraph{Core-set Selection for LLM}
Core-set selection and dataset pruning has a long and rich history in ML research~\citep{har2005smaller, paul2021deep}, where the goal is to find small subsets of training data which gives similar or superior performance as compared to training on the full dataset. A wide range of metrics have been explored for core-set selection, including model loss (e.g. RETRIEVE~\citep{killamsetty2021retrieve}, RHO~\citep{rho}), gradient (e.g. CRAIG~\citep{mirzasoleiman2020coresets}), influence function (e.g. \citep{yang2022dataset}) and clustering. 
\citep{birodkar2019semantic, sorscher2022beyond}. Within the scope of LLMs, prior arts have primarily focused on data selection during pre-training, such as DoReMi~\citep{xie2023doremi}, RHO, DRO~\citep{oren2019distributionally}, and DSIR~\citep{xie2023data}. For post-training LLMs, recent works have focused on selecting training data based on quality, based on either 1) human annotation (e.g. LIMA~\citep{zhou2023lima}), 2) LLM (e.g. AlpaGasus~\citep{chen2023alpagasus}) or 3) validation loss on evaluation dataset (e.g. Instruction Mining~\citep{cao2023instruction}).

%% file: main/method.tex
\subsection{DavIR in Supervised Fine-Tuning}\label{sect:method}
As an ever increasing number of post-training datasets are developed for LLMs {\color{red} }, it is important for practitioners to select a compute-permissible subset of the training data that achieves similar, or \emph{better}, performance than the full available training corpus.

As such, the task that DavIR is set out to solve is one of core-set selection for post-training LLM: given a base model $\pi_\text{base}$, and a collection of training data $D_\text{full} = \{(x_i, y_i)\}_i$ (where $(x_i, y_i)$ represents the prompt/response pair that constitutes a training datum), find a \emph{minimal} subset of the training dataset $D_{train} \subset D_\text{full},~|D_{train}| \ll |D_\text{full}|$ such that the model trained on $D_{train}$ achieves comparable, or {\bf better}, performance than that trained on $D_\text{full}$.

At the core of the DavIR algorithm is the concept of ``learnability''  in post-training LLMs. We are motivated by the ``Superficial Alignment Hypothesis'' \citep{zhou2023lima} which suggests the post-training stage of LLMs involves using a small number of carefully selected training samples to steer a pre-trained LLMs to align with desired response patterns.
In particular, this suggests that the training samples ought to be tightly coupled with the underlying capabilities of the base LLM model, or that such samples need to be ``learnable'' by the base model. 

A simple and intuitive quantification of ``learnability'' is subtracting the evaluation loss of the base model $\pi_\text{base}$ from that of the reference model $\pi_\text{ref}$ ($\pi_\text{base}$ trained on all of $D_\text{full}$):
\begin{equation}
\begin{split}
    &S_{\text{RHO-LM}}(x,y) = \mathcal{L}_\text{base}(y|x) - \mathcal{L}_\text{ref}(y|x) \\
    &= \left[-\log \pi_\text{base}(y|x)\right] - \left[-\log \pi_\text{ref}(y|x)\right].
\end{split}
\label{eqn:rho_lm}
\end{equation} This approach is akin to that of Reducible Holdout Loss (RHO) \citep{rho} and we refer to this vanilla generalization of RHO to causal language modeling as simply RHO-LM.

We note that the RHO-LM metric in Equation~\eqref{eqn:rho_lm} is closely related to the implicit reward function in the Direct Preference Optimization~\citep{rafailov2023direct} procedure. As shown in \citep{rafailov2023direct}, under mild conditions, reward functions $r(x,y)$ consistent with the Bradley-Terry (BT) preference model \citep{bradley1952rank} can be equivalently written as:
\begin{equation}
\begin{split}
    r(x,y) &= \beta \log\frac{\pi(y|x)}{\pi_\text{base}(y|x)}  \\
    &= \beta \cdot\left[\mathcal{L}_\text{base}(x,y) - \mathcal{L}(x,y)\right]
\end{split}
\label{eqn:reward}
\end{equation}
for some language model $\pi(y|x)$ obtained by training via Reinforcement Learning with Human Feedback procedure (using Proximal Poliy Optimization) from a base model $\pi_\text{base}(y|x)$ using the said reward function $r(x,y)$ till optimality.

In other words, the reference model $\pi_\text{ref}(y|x)$ in RHO-LM can be obtained from $\pi_\text{base}(y|x)$ via reward maximization of the \emph{implicit} reward function $r(x,y)$ in Equation \eqref{eqn:reward}. As such, selecting data via RHO-LM using the score function can be viewed as choosing data with maximum reward given by this implicit reward model.

However, we found empirically that the vanilla RHO-LM metric in Equation~\eqref{eqn:rho_lm} is highly correlated with sequence length of the training data (see also Appendix.~\ref{app:A}), an issue that persists despite aggregating the token-level losses via the averaging operation. This is inherently due to the sequential nature of language modeling, where increasing sequence length introduces additional contexts that constraints the distributions of all (subsequent) tokens. The effect of this length dependency is not to be under-estimated, as Table.~\ref{tab:metric_length_corr} shows that correlation between length and average (across tokens) cross-entropy loss as well as entropy of predictive probabilities could be as high as -0.9 (on a scale of [-1, 1]). Consequently, the RHO objective, which subtracts these length-dependent objectives, is also prone to be highly correlated with sequence length (see Table.~\ref{tab:diff_length_corr}) - an issue that only applies to language modeling and was therefore overlooked by the original RHO work which dealt with image classification or NLP tasks with single classification objective (i.e. grammatical correctness in CoLA~\citep{cola} and sentiment analysis in SST-2~\citep{sst2}).

\begin{table}[t]
\resizebox{\linewidth}{!}{%
    \begin{tabular}{@{}ll|llll@{}}
                            &                   & \multicolumn{2}{l}{Pearson} & \multicolumn{2}{l}{Spearman} \\ \midrule
    Dataset                 & Model             & Entropy       & Loss        & Entropy        & Loss        \\ \midrule
    \multirow{4}{*}{Alpaca} & albert-base-v2    & -0.67         & -0.76       & -0.90          & -0.97       \\
                            & bert-base-uncased & -0.64         & -0.79       & -0.88          & -0.95       \\
                            & gemma-2-2b        & -0.66         & -0.61       & -0.80          & -0.76       \\
                            & gemma-2b          & -0.68         & -0.56       & -0.83          & -0.69       \\ \midrule
    \multirow{4}{*}{GSM8K}  & albert-base-v2    & -0.78         & -0.85       & -0.83          & -0.91       \\
                            & bert-base-uncased & -0.73         & -0.57       & -0.77          & -0.90       \\
                            & gemma-2-2b        & -0.75         & -0.72       & -0.79          & -0.74       \\
                            & gemma-2b          & -0.77         & -0.66       & -0.80          & -0.68       \\ \midrule
    \multirow{4}{*}{MBPP}   & albert-base-v2    & -0.69         & -0.83       & -0.64          & -0.85       \\
                            & bert-base-uncased & -0.76         & -0.87       & -0.80          & -0.93       \\
                            & gemma-2-2b        & -0.85         & -0.82       & -0.90          & -0.85       \\
                            & gemma-2b          & -0.85         & -0.82       & -0.91          & -0.83       \\ \bottomrule
    \end{tabular}
}
\caption{
{\bf Language modeling objectives are highly correlated with sequence lenghth.}
Pearson correlation and Spearman rank correlation of entropy and loss with respect to number of tokens in a given document.
Note that all correlations are negative, indicating that token-level entropy/loss decrease as corresponding context length increases.
}
\label{tab:metric_length_corr}
\end{table}

\begin{table*}[t]
    \centering
    \resizebox{\linewidth}{!}{%
    \begin{tabular}{@{}rrrrrrrrrr@{}}
    \multicolumn{1}{l}{\textbf{}} & \multicolumn{1}{l}{\textbf{}} & \multicolumn{2}{c}{\textbf{Alpaca}}                                   & \multicolumn{2}{c}{\textbf{GSM8K}}                                    & \multicolumn{2}{c}{\textbf{LIMA}}                                     & \multicolumn{2}{c}{\textbf{MBPP}}                                     \\ \cmidrule(l){3-10} 
    \multicolumn{1}{l}{\textbf{}} & \multicolumn{1}{l}{\textbf{}} & \multicolumn{1}{l}{RHO-LM} & \multicolumn{1}{l}{DavIR} & \multicolumn{1}{l}{RHO-LM} & \multicolumn{1}{l}{DavIR} & \multicolumn{1}{l}{RHO-LM} & \multicolumn{1}{l}{DavIR} & \multicolumn{1}{l}{RHO-LM} & \multicolumn{1}{l}{DavIR} \\ \midrule
    \multirow{2}{*}{gemma-2b}     & \multicolumn{1}{r|}{spearman $\downarrow$} & 0.75                             & {\bf 0.30}                               & 0.58                             & {\bf 0.06}                               & 0.20                             & {\bf 0.02}                               & {\bf 0.39}                             & 0.62                               \\
                                  & \multicolumn{1}{r|}{pearson $\downarrow$}  & 0.64                             & {\bf 0.33}                               & 0.59                             & {\bf 0.07}                               & 0.23                             & {\bf 0.01}                               & {\bf 0.41}                             & 0.68                               \\ \midrule
    \multirow{2}{*}{gemma-2-2b}   & \multicolumn{1}{r|}{spearman $\downarrow$} & 0.83                             & {\bf 0.47}                               & 0.66                             & {\bf 0.45}                               & 0.17                             & {\bf 0.01}                               & 0.89                             & {\bf 0.39}                               \\
                                  & \multicolumn{1}{r|}{pearson $\downarrow$}  & 0.64                             & {\bf 0.46}                               & 0.65                             & {\bf 0.44}                               & 0.18                             & {\bf 0.10}                               & 0.81                             & {\bf 0.33}                               \\ \bottomrule

    \end{tabular}
    }
\caption{
{\bf DavIR reduces length dependency from the RHO-LM objective.}
Absolute Pearson correlation and Spearman rank correlation of entropy and loss with respect to number of tokens in a given document.
See Appendix.~\ref{app:A} for more detail.
}
\label{tab:diff_length_corr}
\end{table*}

Fortunately, we found that a simple, yet highly effective, normalization technique could dramatically mitigate the length-dependency of the RHO-LM metric, resulting in the normalized score function, which we term DavIR:
\begin{equation}
    \resizebox{0.9\linewidth}{!}{%
        $S_{\text{DavIR}}(x_i, y_i) = \frac{\mathcal{L}_\text{base}(x_i, y_i) - \mathcal{L}_\text{ref}(x_i, y_i)}{\mathcal{L}_\text{base}(x_i, y_i)}$
    }
    \label{eqn:score_norm}
\end{equation}
Note that the denominator in the normalization could be either the base or the reference losses without impacting the ordering of the data via the DavIR metric $S_{\text{DavIR}}$ (see Appendix.~\ref{app:B} for a simple proof). The reduction in both spearman and pearson correlation is shown in Table.~\ref{tab:diff_length_corr}.

% {\color{red} comment on the inherent nature of length dependency in causal sequence modeling.}

Given the DavIR score function, the DaVIR algorithm for supervised fine-tuning data selection is simply given as Algorithm~\ref{alg:davir}.

\begin{algorithm}
\caption{DavIR for Supervised Fine-tuning}
\begin{algorithmic}[1]
\State $\pi_\text{ref}(y|x) \leftarrow \pi_\text{base}(y|x)$ trained on $D_\text{full}$
\For{each $(x_i, y_i) \in D_\text{full}$}
    \State $\mathcal{L}_\text{base}(x_i, y_i) \leftarrow -\log\pi_\text{base}(y_i|x_i)$
    \State $\mathcal{L}_\text{ref}(x_i, y_i) \leftarrow -\log\pi_\text{ref}(y_i|x_i)$
    \State Compute $S_{\text{DavIR}}(x_i, y_i)$ as in Equation \eqref{eqn:score_norm}
\EndFor
\State Re-train $\pi_\text{base}$ on $\text{top-}k_{D_\text{full}} ~S_\text{DavIR}(x_i, y_i)$
\end{algorithmic}
\label{alg:davir}
\end{algorithm} 

As we later demonstrate, while vanilla RHO-LM is effective in selecting a subset of the training data, it far under-performs the length-regularized DavIR algorithm in the downstream tasks performances across multiple datasets and models (see  Figure.~\ref{Fig.3.5-winscore}).
In fact, as demonstrated in Table~\ref{tab:davir_vs_everyone}, DavIR is able to outperform all (to the best knowledge of the authors) existing core-set selection techniques on post-training LLMs.

Finally, we note that both RHO-LM and DavIR score functions capture the essence of training on data that are ``learnable, worth learning, and not yet learnt''~\citep{rho}, while dramatically reducing the length-dependencies of the original RHO objective. By focusing on the same next-token-prediction objective as training LLM, and omitting confounding factors such as additional small proxy models~\citep{xie2023doremi} or hold-out dataset~\citep{rho}, DavIR provides exact single datum-level measurement of ``learnability'' that tightly couples with the underlying capabilities of the pre-trained model.

\subsection{DavIR in Direct Preference Optimization}
The performance gain of DavIR over the vanilla RHO-LM motivated use to revisited the DPO training objective and the underlying BT preference model. In particular, we propose a simple generalization of the DPO objective with normalization from the reference model loss. In particular, inspired by the formula in Equation.~\ref{eqn:score_norm}, we propose the following DavIR-DPO loss:
\begin{equation}
\begin{split}
    &\mathcal{L}_{\text{DavIR-DPO}}(\pi_\theta; \pi_{\text{ref}}) \\
    &= -\mathbb{E}\Bigg[ \log \sigma \Big( \beta \log \frac{\pi_\theta(y_w \mid x)}{\pi_{\text{ref}}(y_w \mid x)} \big / | \log \pi_{\text{ref}}(y_w \mid x) | \\
    &- \beta \log \frac{\pi_\theta(y_l \mid x)}{\pi_{\text{ref}}(y_l \mid x)} / | \log \pi_{\text{ref}}(y_l \mid x) | \Big) \Bigg].
\end{split}
\label{eqn:davir_dpo}
\end{equation}
We remark that concurrent research on regularizing the DPO loss by the length of the responses has been proposed \citep{park2024disentangling}. 

\begin{table}[t]
\resizebox{\linewidth}{!}{%
\begin{tabular}{@{}llc@{}}
DPO Type       & Reference & Resp. Len. Diff. \\ \midrule
Vanilla        & \citep{rafailov2023direct}          & 0.38                \\
AOT            & \citep{dpo_aot}          & 0.12                \\
APO (Down)     &  \citep{dpo_apo}          & 0.36                \\
APO (Zero)     &  \citep{dpo_apo}           & 0.39                \\
EXO (Pair)     & \citep{dpo_exo}           & 0.42                \\
Hinge          & \citep{dpo_rso}           & 0.39                \\
IPO            & \citep{dpo_ipo}          & -0.10               \\
NCA            &  \citep{dpo_nca}          & 0.29                \\
Robust         & \citep{dpo_robust}           & 0.38                \\
SPPO (Hard)    &  \citep{dpo_sppo}          & -0.11               \\
\textbf{DavIR} &  Here         & \textbf{0.07}       \\ \bottomrule
\end{tabular}
}
\caption{{\bf Pearson correlation of DPO objective against difference in response length for different flavors of DPO loss type in UltraFeedback~\citep{ultrafeedback}}}
\label{tab:dpo_length_corr}
\end{table}

% %% Fig overview
% \begin{figure}[!t]
% \centering 
% \includegraphics[width=0.95\linewidth]{fig/overview.pdf} 
% \caption{
% DaVIR method. We start from a pre-trained model, e.g.LLaMA and mixed SFT dataset, e.g. Alpaca. \textbf{Reference Model Training}: the initial model is fine-tuned with the full dataset to get the reference model. \textbf{Data Selection}: both loss of reference model and initial model is used to compute the score of each data point and the score is then ranked for selection.} 
% \label{Fig.overview}
% \end{figure}

%% file: main/exp.tex
\subsection{Experimental Setup}
\paragraph{Training Dataset.}
Training datasets used in the current study are shown in Table.~\ref{tab:train_data}. Note that both Alpaca-4 and Alpaca-3.5 were proposed in \citep{alpaca}, with the same prompts but different responses generated by GPT-4 and GPT-3.5-Turbo respectively.

\begin{table}[!t]
\resizebox{\linewidth}{!}{%
\begin{tabular}{@{}lll@{}}
\toprule
Name        & Number & Source \\ \midrule
LIMA~\citep{zhou2023lima}                  & 1K     & Human       \\
Alpaca-4~\citep{alpaca}             & 52K    & GPT-4       \\
Alpaca-3.5~\citep{alpaca}             & 52K    & GPT-3.5-Turbo     \\
GSM8K~\citep{cobbe2021training}                 & 7.5K   & Human       \\ 
UltraFeedback~\citep{ultrafeedback}   & 61K & GPT-4 \\\bottomrule
\end{tabular}
}
\caption{{\bf Training datasets used for experimental validation of DavIR.}}
\label{tab:train_data}
\end{table}

\paragraph{Test Dataset and Evaluation Method.} \quad
For open-domain freeform QA style evaluation of LLaMA models, our test set is an amalgamation of 800 prompts from HH-RLHF, Koala, Self-Instruct, Open Assistant, and Vicuna, covering multiple aspects of daily use, such as generating, math, coding, and instruction-following. The model generated responses were evaluated either with GPT-4 (adjusted for positional bias in evaluation prompt) or human evaluator (blind ranking) as referee. Note that only 100 questions were randomly selected for human evaluation (20 questions per dataset). The performance of models trained with DavIR filtered dataset is compared against either 1) same based model trained with other data selection method as in Fastchat~\citep{zheng2023judging}, or 2) against frozen model (e.g. Text-Davinci-003) as in AlpacaEval ~\citep{alpaca_eval}. 
Experiments with Gemma and Zephyr models were evaluated using AlpacEval2.0~\citep{alpaca_eval_v2}.

\paragraph{Models and Baselines.}
For IFT/SFT experiments, we used LLaMA-7B, LLaMA-13B \citep{touvron2023llama}, Gemma-2B~\citep{gemma} models as our base models $\pi_{base}$. DavIR is compared against a wide range of baseline data selection methods (See Table.~\ref{tab:other_methods}): 1) full dataset, 2) random sampling, 3) ChatGPT-based data filtering~\citep{chen2023alpagasus}, 4) RHO-LM\citep{rho}, 5) EL2N~\citep{el2n}, 6) Forgetting score ~\citep{forgetting}, and 7) Influence function-based DataInf~\citep{datainf}. Note that, for ChatGPT-based data filtering approach, a specific version of ChatGPT API was prompted to assign an integer quality score (1-5) to each data point in the Alpaca-3.5 dataset. To avoid introducing additional variabilities due to changes of ChatGPT API, for comparison against ChatGPT-based data filtering, we did not generate new data by querying ChatGPT, but instead directly used the 9k subset of Alpaca-3.5 reported in \citep{chen2023alpagasus}. The comparison with which is shown in Figure.~\ref{fig:fig1} where ChatGPT-based data filtereing is referred to as ``GPT-Select''. DPO experiments were conducted using Zephyr-7B-SFT~\citep{Tunstall_The_Alignment_Handbook} and compared against training with the vanilla DPO objective.

\begin{table*}[t]
\resizebox{\textwidth}{!}{%
    \begin{tabular}{llllll}
    \toprule
    \textbf{Name} & \textbf{Model Dep.} & \textbf{Val. Set} & \textbf{Method/Metric}                                    & \textbf{Domain}      & \textbf{Metric Selection} \\
    \midrule
    EL2N          & Loss                      & No               & $\ell_2$ Loss between logits and 1-hot label                         & Image Class. & Larger                        \\
    Forgetting    & Loss                      & No               & \# epochs were data loss increase                         & Image Class. & Larger                               \\
    DataInf       & Gradient                  & Yes              & Approx. influence function                      & Text Generation      & Most Negative                        \\
    AlpaGasus     & ChatGPT                  & No                & ChatGPT scores 0-5                          & Text Generation      & Larger                               \\
    RHO           & Loss                      & Yes              & Eq.~\eqref{eqn:rho_lm} & Image Class./NLP Class. & Larger                               \\
    DavIR         & Loss                      & No               & Eq.~\eqref{eqn:score_norm}   & Text Generation      & Larger                              
    \\ 
    \bottomrule
    \end{tabular}
}
\caption{{\bf Baseline core-set selection methods.}}
\label{tab:other_methods}
\end{table*}

\subsection{DavIR in SFT}
\subsubsection{Impact of Length Normalization in DavIR}
We first demonstrate the effect of normalization in the DavIR objective as compared to the RHO-LM objective. As shown in Figure.~\ref{Fig.3.5-winscore}, while LLaMA models trained on subset of data selected using the RHO-LM objective can achieve comparable performance to that of the model trained on the full dataset, DavIR outperforms the full dataset baseline by a wide margin.

\begin{figure}[!t]
\centering 
\includegraphics[width=\linewidth]{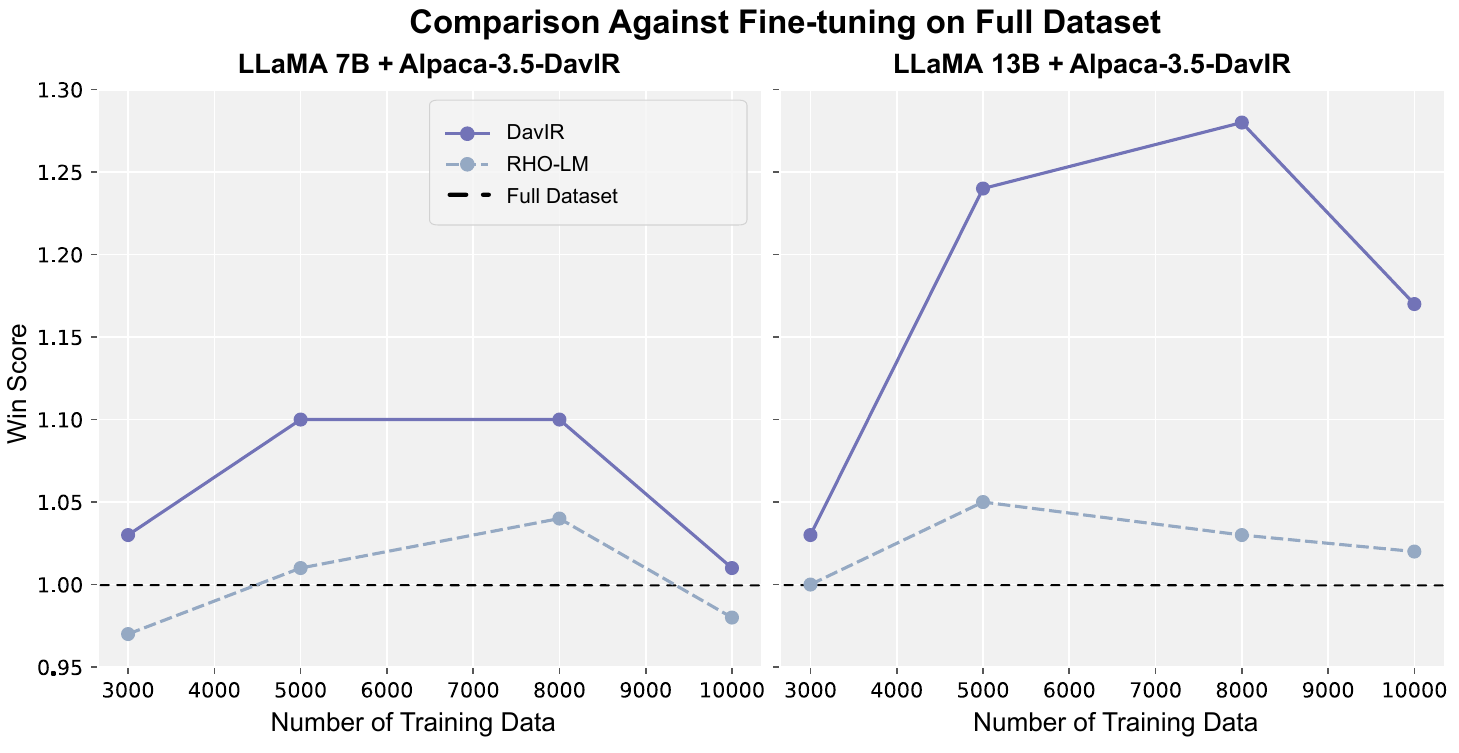} 
\caption{\textbf{Models fine-tuned with data selected by DavIR surpass the full dataset on Alpaca3.5.} This figure shows the win score comparison between models trained with different sizes of datasets and the full dataset, as well as the improvement brought by using the normalization method. We select the model fine-tuned on the full dataset as the baseline. Win Score is computed as $1+(N_{win}-N_{lose})/N_{total}$, with $1$ being equal performance.}
\label{Fig.3.5-winscore}
\end{figure}

\subsubsection{16x Compression in Freeform Chat Dataset}
As show in Figure.~\ref{fig:fig1}, both LLaMA-7B and LLaMA-13B model can be effectively fine-tuned with  a 3K subset sampled from the 52K Alpaca dataset using DavIR. A natural question to ask is whether this is a result of simply reducing redundancy in the training dataset, which could also be achieved by simply randomly sampling the dataset. To address this question, we compared performance of DavIR against random sampling and fine-tuning on full Alpaca-4 dataset using Text-Davinci-003 as a frozen baseline model. We show in Figure.~\ref{Fig.4-winrate} that the number of training data, when randomly sampled, improve model performance logarithmically, dramatically under-performing the proposed method.

\begin{figure}[!t]
\centering 
\includegraphics[width=\linewidth]{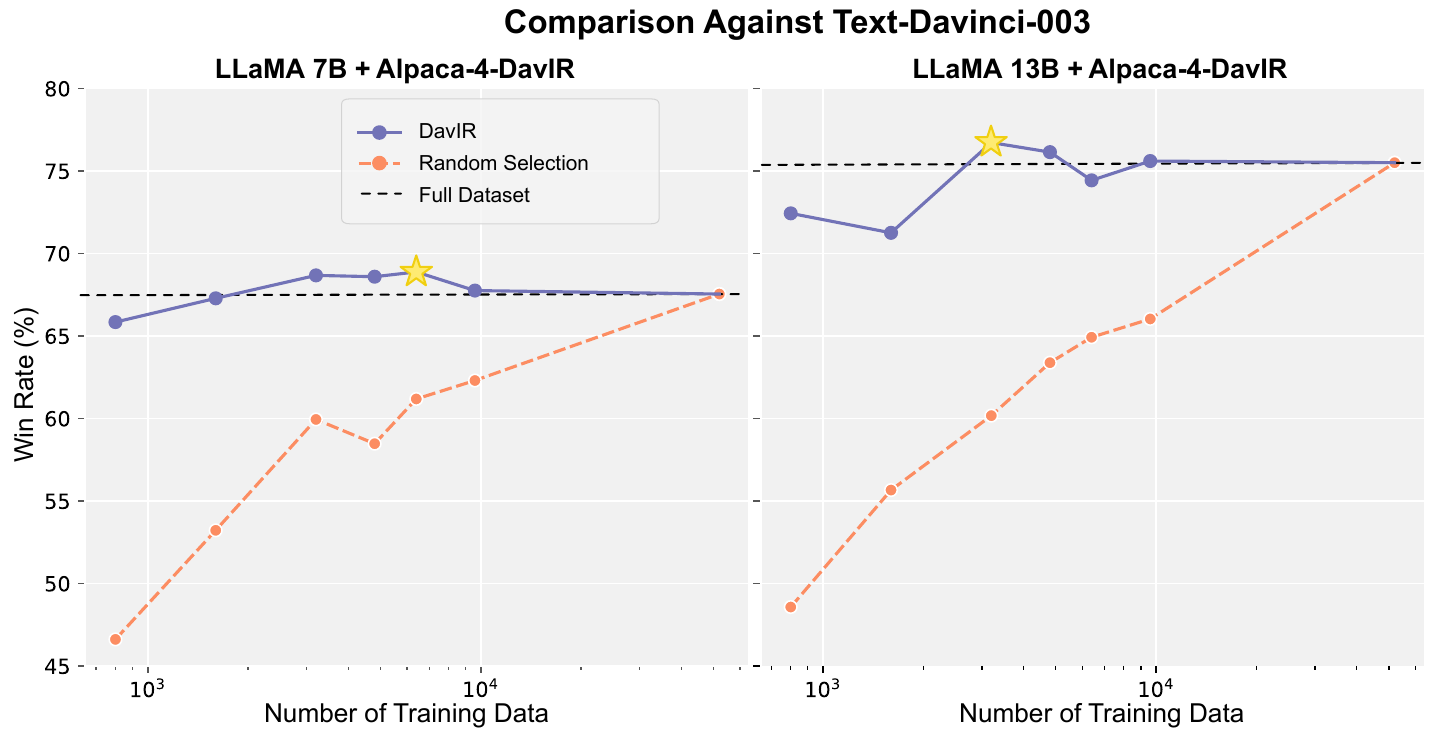} 
\caption{\textbf{DavIR significantly out perform random sampling.} Using Text-Davinci-003 as the frozen baseline model, we show that performance of random selection of the Alpaca-4 dataset scales logarithmically with number of training data, significantly under-performing DavIR. Note that the x-axis is log-scale. Win Rate is computed as $N_{win}/N_{total}$, where $N_{win}, N_{total}$ are number of win and total number of test data.
} 
\label{Fig.4-winrate}
\end{figure}

To further demonstrate DavIR's effectiveness against other baseline methods, we compared DavIR's performance scaling across number of training data (selected from Alpaca dataset) against 4 other core-set selection methods (EL2N, Forgetting Score, DataInf, RHO), evaluated on Gemma-2B model using AlpacaEval. As shown in Table.~\ref{tab:davir_vs_everyone}, DavIR is the only method that consistently out-performs full dataset baseline across number of training samples. Even in the low data regime (less than 5K selected from the 52K dataset) when DavIR is not the best performing method, its performance gap with the best performing method is small.
We'd also like to emphasize that computing the DavIR score requires only computing validation losses. In contrast, the second best performing algorithm (DataInf) requires significantly more compute as it requires computing Influence Functions via gradient and approximated Hessian. To show that the performance gain of DavIR compared against other methods is \emph{statistically} significant, we estimated the 95\% confidence interval of the AlpacaEval score via bootstrap sampling (shown in Table.~\ref{tab:gemma2b_with_CI}). Since the bootstrap sampled distribution of AlpacaEval score is highly Gaussian, we performed t-test between the sample distributions of AlpacaEval scores, will DavIR beating almost all baseline methods across number of data with very low p-values, providing conclusive evidence of the effectiveness of DavIR. Refer to Appendix.~\ref{sect:statistical_analysis} for more details on statistical analysis.

\begin{table*}[t]
\centering
\resizebox{.7\linewidth}{!}{%
    \begin{tabular}{@{}llllllllll@{}}
    \toprule
    \textbf{Method}             & \textbf{Selection} & \textbf{3000} & \textbf{4000} & \textbf{5000} & \textbf{6000} & \textbf{7000} & \textbf{8000} & \textbf{9000} & \textbf{10000} \\ \midrule
    Random                      & N/A                & 10.6          & 12.7          & 15.9          & 16.1          & 17.0          & 16.3          & 16.7          & 17.6           \\
    \multirow{2}{*}{EL2N}       & Lowest             & 10.3          & 13.5          & \textbf{17.6} & 17.6          & 19.0          & 18.4          & 18.7          & 18.5           \\
                                & Highest            & 10.0          & 11.1          & 11.3          & 11.8          & 12.4          & 13.3          & 13.9          & 14.3           \\
    \multirow{2}{*}{Forgetting} & Lowest             & 10.4          & 12.8          & 16.7          & 16.6          & 17.1          & 17.0          & 17.7          & 16.7           \\
                                & Highest            & 9.5           & 12.4          & 13.4          & 15.5          & 16.7          & 17.8          & 18.9          & 18.2           \\
    \multirow{2}{*}{DataInf}    & Lowest             & 10.3          & 12.2          & 13.9          & 15.0          & 15.0          & 14.4          & 15.3          & 15.4           \\
                                & Highest            & 10.3          & \textbf{13.7} & 15.9          & 18.6          & 18.7          & 19.8          & 18.5          & 18.8           \\
    \multirow{2}{*}{RHO}        & Lowest             & 9.6           & 12.1          & 13.7          & 15.3          & 15.7          & 16.1          & 16.2          & 16.6           \\
                                & Highest            & 9.9           & 12.3          & 14.5          & 15.8          & 15.3          & 15.5          & 16.5          & 16.7           \\
    \multirow{2}{*}{DavIR}      & Lowest             & 10.6          & 12.7          & 14.2          & 14.3          & 13.1          & 12.0          & 13.2          & 13.6           \\
                                & Highest            & \textbf{10.8} & 13.1          & 17.1          & \textbf{20.2} & \textbf{20.2} & \textbf{20.4} & \textbf{20.2} & \textbf{19.7}  \\
    Full                        & N/A                & \textit{18.3} & \textit{18.3} & \textit{18.3} & \textit{18.3} & \textit{18.3} & \textit{18.3} & \textit{18.3} & \textit{18.3}  \\ \bottomrule
    \end{tabular}
}
\caption{{\bf Comparing DavIR to baselines for post-training Gemma-2B model on Alpaca dataset.} Performance reported here is the AlpacaEval win-rate against GPT-4. For completeness, we included the performances of Gemma-2B trained on Alpaca subsets selected by choosing both lowest and highest metric values (EL2N, Forgettting, DataInf, RHO, DavIR). Note that comparison against AlpaGasus is shown in Fig.~\ref{fig:fig1}, and is omitted here because we only have access to the full 9K+ subset reported by the authors of \citep{chen2023alpagasus} which limits our ability to perform ablation across number of data points. Refer to Appendix.~\ref{sect:statistical_analysis} for statistical analysis of the performance comparison between DavIR and other baseline methods.
}
\label{tab:davir_vs_everyone}
\end{table*}

Finally, to demonstrate the effectiveness of DavIR on other dataset, we performed data subset selection on the LIMA dataset. We show that even for a carefully curated dataset (1K), DavIR is still able to achieve a 3x compression (300) while achieving comparable performance to training on the full dataset (win score 1.01).

\subsubsection{Balancing Open-Domain QA and Mathematical Reasoning}
A key application of core-set selection in LLM in production is data flywheel scenario, where a constant stream of additional training data for multiple domdains need to be filtered and combined to produce the best and most-balanced model for a wide range of downstream tasks.
To that end, we evaluated the performance of LLaMA-7B model trained on a combination of the full GSM8K dataset (mathematical reasoning) and DavIR-fitlered Alpaca-4 subset (freeform QA). As shown in Figure.~\ref{Fig.mix}, while we do observe the issue of ``alignment tax''~\citep{casper2023open} where increasing Alpaca-4 data size caused a slight decrease in GSM8K accuracy, DavIR offers the flexibility for LLM developers to control the balance of open-domain QA with mathematical reasoning capabilities. In particular, the addition of 3.2K Alpaca-4 data (using 16.7\% of total trainig data) boosts open-domain QA performance from <10\% win-rate to >60\% win-rate, at the cost of 2\% reduction of GSM8K accuracy as compared to the model trained solely on GSM8K training set.

\begin{figure}[!t]
\centering 
\includegraphics[width=0.8\linewidth]{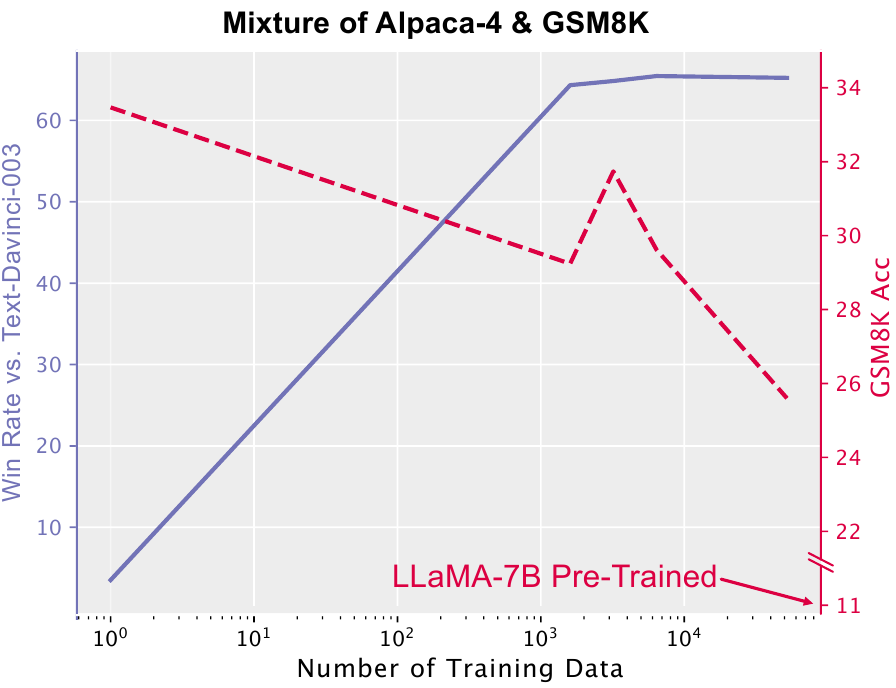} 
\caption{\textbf{Data mixing with LLaMA-7B and DavIR.} The x-axis represents the number of selected Alpaca-4 data points, plotted on a logarithmic scale.} 
\label{Fig.mix}
\end{figure}

\subsubsection{Generalization between Models}
As shown above, DavIR is highly effective across both model sizes (LLaMA-7B/-13B) as well as model families (LLaMA, Gemma). 
However, as DavIR was fundamentally motivated by the hypothesis that the best post-training data must be model-dependent, we sought to  examine the data selected by different models. Comparing the best data subset selected using LLaMA-7B and LLaMA-13B models, we observe that only 516 of the top 800 data are scored highly by both models (see Table.~\ref{tab.same} in Appendix.~\ref{app:C}), with the largest difference stemming from mathematical reasoning-related prompts as shown in Table.~\ref{tab.cata} in Appendix.~\ref{app:C}. Given that both models belong to the LLaMA family of model, and were presumably trained using similar training recipes (architecture, hyperparameters, datasets), we expect the discrepancy between data subsets selected by different models to widen when comparing between models of different families. This provides support for our intuition that the effectiveness in steering pre-trained model is highly dependent on the capabilities of the pre-trained model itself.

To probe the model-dependency of the post-training data selection from a different perspective, 
we explored a relaxed version of the DavIR algorithm. In this relaxed version, the base, reference model and re-trained models are allowed to be from different pre-trained models. For example, instead of using the LLaMA-7B base model throughout the data selection and re-training process, we experimented with computing DavIR score between LLaMA-7B base model and LLaMA-13B model trained on all of Alpaca dataset, and re-trained LLaMA-7B base model on the selected subset (and similarly for other combinations of base, reference and re-trained models). As shown in Table.~\ref{tab:generalization}, any misalignment in the models used in DavIR algorithm resulted in a decrease in performance, providing further support to the model-dependency of the optimal post-training data subset.

\begin{table}[t]
\resizebox{\linewidth}{!}{%
\begin{tabular}{@{}llllll@{}}
\toprule
$\pi_{base}$      & $\pi_{ref}$            & $\pi_{retrain}$  & Against     & Data & Win Score \\ \midrule
13B & 13B + Alp-4 & 13B & 13B + Alp-4 & 8K   & {\bf 1.2}      \\
7B  & 7B + Alp-4  & 7B  & 7B + Alp-4  & 8K   & {\bf 1.125}    \\\midrule
13B & 13B + Alp-4 & 7B  & 7B + Alp-4  & 8K   & 0.825    \\
7B  & 13B + Alp-4 & 7B  & 7B + Alp-4 & 8K   & 1.01     \\ \bottomrule
\end{tabular}
}
\caption{{\bf DavIR performs best when base, reference and re-trained models share the same pre-trained backbone.}
For simplicity, 7B/13B refer to LLaMA-7B and LLaMA-13B respectively, $\pi_{retrain}$ refers to the pre-trained model that is trained with the DavIR selected data subset. Win Score is computed against models fine-tuning on $D_\text{full}$ as shown in the ``Against'' column. Note that the first two rows correspond to experiments where there is no model mismatch and the bottom tow rows correspond to the relaxed DavIR algorithm with model mismatch.
}
\label{tab:generalization}
\end{table}

\subsection{DavIR in DPO}
We trained Zephyr-7B-SFT~\citep{Tunstall_The_Alignment_Handbook} on UltraFeedback~\citep{ultrafeedback} paired preference dataset using both vanilla DPO objective as well as the DavIR-DPO objective in Equation.~\ref{eqn:davir_dpo}. Both models are evaluated on AlpacaEval againest the text-davinci-003 model. We present the result in Table.~\ref{tab:results_dpo} show that Zephyr trained using the DavIR-DPO objective outperforms that using the vanilla DPO objective, especially when evaluated using length-controlled metric~\citep{alpaca_eval_v2}. Note that Zephyr model class was chosen for the DPO experiments, as opposed to LLaMA and Gemma as in the SFT experiments, for it differentiates between pretrained, instruction fine-tuned and DPO post-trained models, thus helping us isolate the effect of length-normalization at the DPO training stage.

\begin{table}[t]
\resizebox{\linewidth}{!}{%
\begin{tabular}{@{}lp{3cm}c@{}}
                 & Length-Controlled Win-Rate & Win-Rate \\ \midrule
Zephyr DPO       & 57.23                      & 82.31    \\
Zephyr DavIR-DPO & {\bf 61.83}                      & {\bf 82.96}    \\ \bottomrule
\end{tabular}
}
\caption{{\bf Comparing Zephyr trained on Davir-DPO vs. vanilla DPO objective.}}
\label{tab:results_dpo}
\end{table}

%% file: main/conclusion.tex
 We introduce DavIR, a model-based data selection method for LLM fine-tuning that focuses on ``learnability'' of data points given a base pre-trained model. We show that DavIR is closely related to, and is a generalization of, the Implicit Reward Model concept proposed in Direct Preference Optimization. By comparing DavIR to a wide range of data selection baselines, we demonstrate its effectiveness across models, data domain and data mixtures. Finally, we show that, by incorporating the proposed normalization back to the DPO objective, we are able to improve DPO performance after the supervised fine-tuning stage of LLM training.

%% file: main/discussion.tex
\paragraph{Integration of DavIR to Data Flywheel}
As briefly discussed above, data compression techniques such as DavIR serve a critical, albeit incomplete, role in the data flywheel of training LLMs. In particular, DavIR does not take into account other aspects of data selection such as quality and diversity. In practice, DavIR needs to be used in conjunction with methods such as weighted sampling and prompt classification to ensure that the core-set selection is performed in a manner that does not artificially bias the distribution of the selected data. In this work, we provided a simple example of data mixture between Alpaca and GSM8K which hints at the importance of using DavIR in a manner that is conscious of data diversity. In production, however, more careful design of the data pipeline is required, for which DavIR could serve as the data compression module.

\paragraph{Application of DavIR to Reasoning Tasks}
A key limitation of DavIR is that its effectiveness varies based on the application domain of the training dataset. In particular, when we applied DavIR to compressing the GSM8K training dataset alone for LLaMA models, we did not observe a clear performance gain with subset of the training data. In fact, as shown in Figure.~\ref{fig:gsm8k}, we observed almost linear scaling of number of GSM8K training sample and the GSM8K evaluation accuracy, suggesting that the GSM8K dataset was \emph{in-compressible} with LLaMA-7B using DavIR. We hypothesize that this could be caused by LLaMA-7B having insufficient underlying mathematical reasoning capabilities, leading to very large training data requirement. However, we could not rule out the possibility that perhaps Cross-Entropy Loss is a poor metric for how well data related to mathematical reasoning has been learnt by a given model, thereby rendering the normalized score metric unable to capture ``learnability'' of such data. We leave explorations of alternate metrics to Cross-Entropy Loss for future works.

\begin{figure}[h]
    \centering
    \includegraphics[width=\linewidth]{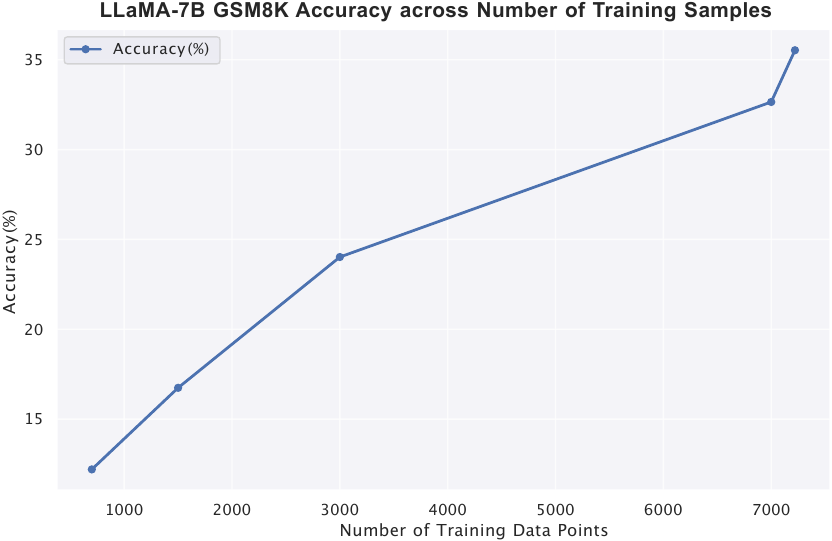}
    \caption{{\bf GSM8K training data was in-compressible with LLaMA-7B.}}
    \label{fig:gsm8k}
\end{figure}

% Additionally, while we've provided support for the claim that SFT data selection is highly dependent on the 
% Another limitation and future direction of our work is that DavIR is primarily focused on data selection within a single domain/dataset (e.g. Alpaca). In practice, DavIR could be used as a component of the data flywheel where a larger collection of datasets covering a wide range of domains are independently or jointly compressed to both improve training efficiency and model performance. However, as data from different domains have naturally different loss distributions (e.g. coding related SFT data have low training losses due to code having highly predictably syntactically content), adapting DavIR to simultaneously support multiple domains require further development of our method.

%% file: main/ethics.tex
This research aims to provide a model-based algorithm for core-set selection of LLM alignment training data.
Experimental validation in the current work leverages previously published datasets, and are employed in accordance with their intended use cases.
While these datasets are widely used, we acknowledge that we cannot fully ascertain the extent to which they may contain discriminatory, biased, or sensitive material.

\paragraph{Responsible Usage:} Data selection via DavIR is based purely on model’s perceived degree of understanding, and makes no assumption about safety of the original training data. 
As such, caution must be exercised when deploying DavIR in production to ensure that necessary safety practices are adopted both before and after using DavIR for subset selection.

\paragraph{AI Assistant Usage:} Claude-3.5-Sonnet and GPT-4o were used for grammatical correction in the current manuscript.

%% file: main/appendix.tex
\section{Details on Human Evaluation of Model Performance}\label{app:human}
For open-domain freeform QA style evaluation of LLaMA models (an amalgamation of 5 test datasets), we used a combination of LLM (GPT-4) and human as referee. For human evaluation, 20 questions per test dataset were randomly selected as prompts, resulting in total 100 prompts. One human annotator (unpaid, college educated, age 20-25, proficient in English) was provided side-by-side comparison of two responses generated by two models for each question and asked to determine whether the response on the ``left'' is better/same/worse (win/tie/lose) than the response on the ``right''. The annotator is blind to the identity of the model for which the responses were generated, and the responses were randomly ordered.

\section{Effect of Normalization for Score Function}\label{app:A}
% To verify whether this problem exists, we observe the relationship between the loss and score of selected 6400 data points in $D_\text{ori}$ , as shown in Figure \ref{Fig.loss-score}. We can observe that with high loss values are ranked at the forefront. 

% \begin{figure}[h] 
% \centering 
% \includegraphics[width=\linewidth]{fig/loss-score.png} 
% \caption{The horizontal axis of this figure is the score ranking, the closer to the left, the smaller the ranking, that is, the larger the score; the vertical axis is the value of the loss. The left figure shows the loss of $M_\text{ini}$, and the right figure shows the loss of $M_\text{ref}$.} 
% \label{Fig.loss-score}
% \end{figure}

% We further analyze the cause of this phenomenon and believe that it may be caused by the calculation method of loss in the fine-tuning process. Because in the fine-tuning process, we mask the tokens corresponding to the prompt for the calculation of loss and only calculate the loss of the tokens corresponding to the response, specifically, we calculate the average cross-entropy loss of each token of the response. Therefore, we believe that the size of the loss may be related to the length of the response. 
In Figure.~\ref{Fig.len-score}, we compare the sequence length of the Alpaca dataset ranked by either un-normalized and normalized score functions. It is apparent that without normalization, data with the highest scores(low ranking) correspond to data with very short sequence lengths. In contrast, the introduction of normalization completely removes the 

\begin{figure}[h]
\centering 
\includegraphics[width=\linewidth]{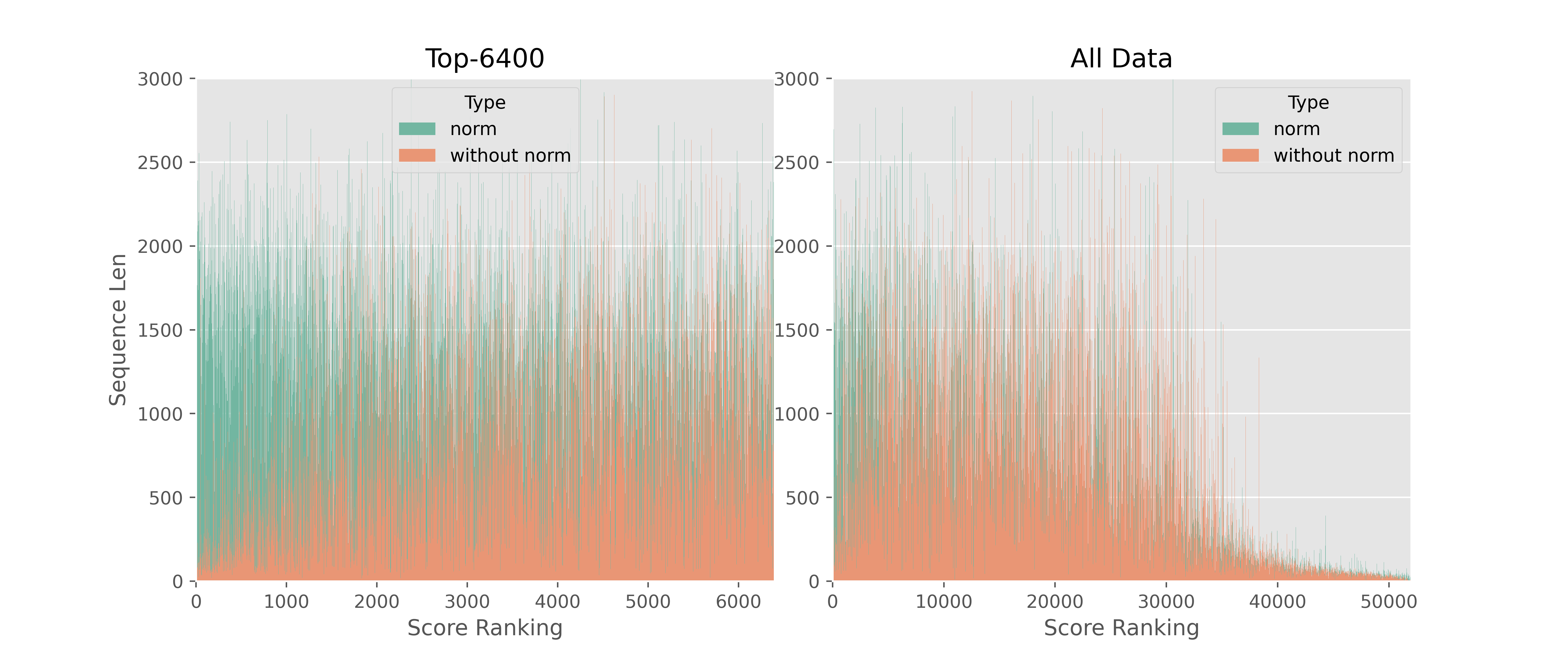} 
\caption{{Effect of normalization on score function and sequence length.} The relationship between the sequence length and ranking for both with and without normalizing score function is shown for the (left) top-6400 subset and (right) full 52K Alpaca-4. We observe that, without normalization, data with highest score values have noticeably short sequence length, which is resolved by normalization.} 
\label{Fig.len-score}
\end{figure}

\section{Choice of Denominator in Normalized Score Function Does Not Impact Ranking}\label{app:B}

\begin{proposition}
Choosing either $\mathcal{L}_{ref}(x, y)$ or $\mathcal{L}_{base}(x, y)$ as the denominator for normalization does not affect the ranking of the learnability score. Specifically, if 
\begin{equation}
\resizebox{\linewidth}{!}{$
\frac{\mathcal{L}_{base}(x_1, y_1) - \mathcal{L}_{ref}(x_1, y_1)}{\mathcal{L}_{base}(x_1, y_1)} > \frac{\mathcal{L}_{base}(x_2, y_2) - \mathcal{L}_{ref}(x_2, y_2)}{\mathcal{L}_{base}(x_2, y_2)}
$}
\end{equation}
then it also holds that
\begin{equation}
\resizebox{\linewidth}{!}{$
\frac{\mathcal{L}_{base}(x_1, y_1) - \mathcal{L}_{ref}(x_1, y_1)}{\mathcal{L}_{ref}(x_1, y_1)} > \frac{\mathcal{L}_{base}(x_2, y_2) - \mathcal{L}_{ref}(x_2, y_2)}{\mathcal{L}_{ref}(x_2, y_2)}
$}
\end{equation}
\end{proposition}

\begin{proof}
Assume that 
\begin{equation}
\resizebox{\linewidth}{!}{$
\frac{\mathcal{L}_{base}(x_1, y_1) - \mathcal{L}_{ref}(x_1, y_1)}{\mathcal{L}_{base}(x_1, y_1)} > \frac{\mathcal{L}_{base}(x_2, y_2) - \mathcal{L}_{ref}(x_2, y_2)}{\mathcal{L}_{base}(x_2, y_2)}
$}
\end{equation}
which can be rewritten as
\begin{equation}
1 - \frac{\mathcal{L}_{ref}(x_1, y_1)}{\mathcal{L}_{base}(x_1, y_1)} > 1 - \frac{\mathcal{L}_{ref}(x_2, y_2)}{\mathcal{L}_{base}(x_2, y_2)}.
\end{equation}
This implies
\begin{equation}
\frac{\mathcal{L}_{ref}(x_1, y_1)}{\mathcal{L}_{base}(x_1, y_1)} < \frac{\mathcal{L}_{ref}(x_2, y_2)}{\mathcal{L}_{base}(x_2, y_2)}.
\end{equation}
Taking the reciprocal of both sides, we get:
\begin{equation}
\frac{\mathcal{L}_{base}(x_1, y_1)}{\mathcal{L}_{ref}(x_1, y_1)} > \frac{\mathcal{L}_{base}(x_2, y_2)}{\mathcal{L}_{ref}(x_2, y_2)}.
\end{equation}
Subtracting one from both sides, we obtain:
\begin{equation}
\resizebox{\linewidth}{!}{$
\frac{\mathcal{L}_{base}(x_1, y_1) - \mathcal{L}_{ref}(x_1, y_1)}{\mathcal{L}_{ref}(x_1, y_1)} > \frac{\mathcal{L}_{base}(x_2, y_2) - \mathcal{L}_{ref}(x_2, y_2)}{\mathcal{L}_{ref}(x_2, y_2)}.
$}
\end{equation}
Thus, we have shown that normalizing by either $\mathcal{L}_{base}$ or $\mathcal{L}_{ref}$ does not affect the ranking of the learnability score.
\end{proof}

\section{Analysis of Data Selected via DavIR}\label{app:C}
To explore what types of data are required by the model during the SFT process, we conducted a further analysis of the data selected by the 7B and 13B models. In Table \ref{tab.same}, we show the number and percentage overlap of the data points selected by both LLaMA-7B and LLaMA-13B models. 

\begin{table}[hbt]
\centering
\renewcommand\arraystretch{1.2}
\resizebox{\linewidth}{!}{%
\begin{tabular}{llllllll}
\noalign{\hrule height 1pt}
No. Data & 800 & 1,600 & 3,200 & 6,400 & 9,600 & 26,000 & 39,000 \\
\noalign{\hrule height 1pt}
No. Overlap & 516 & 1,111 & 2,399 & 5,123 & 8,033 & 24,358 & 38,075 \\
Percent Overlap & 64.5\% & 69.4\% & 74.9\% & 80.0\% & 83.7\% & 93.7\% & 97.6\% \\
\noalign{\hrule height 1pt}
\end{tabular}
}
\caption{{\bf The number and percentage overlap of data points selected by LLaMA-7B and LLaMA-13B}.}
\label{tab.same}
\end{table}

\paragraph{Constituency Parsing via Benepar.}
We used the Benepar to parse constituency of the top 800 data points selected by the LLaMA-7B and LLaMA-13B models, which have 516 overlapping data points as shown in Table.~\ref{tab.same}. Benepar decomposes natural language statements into hierarchical representation of constituency, from which we visualize the top two level verb predicate and noun objects. 

Upon close examination of Figure.~\ref{fig:7Bdata} and Figure.~\ref{fig:13Bdata}, we observe that, comparing the more powerful LLaMA-13B against LLaMA-7B, fewer creation tasks (e.g. write, generate, create) and more interpretative tasks (e.g. explain, describe) data were selected, with slighter more diverse long tail tasks.

\begin{figure*}[h]
  \centering
  \begin{subfigure}[b]{0.49\linewidth}
    \includegraphics[width=\linewidth]{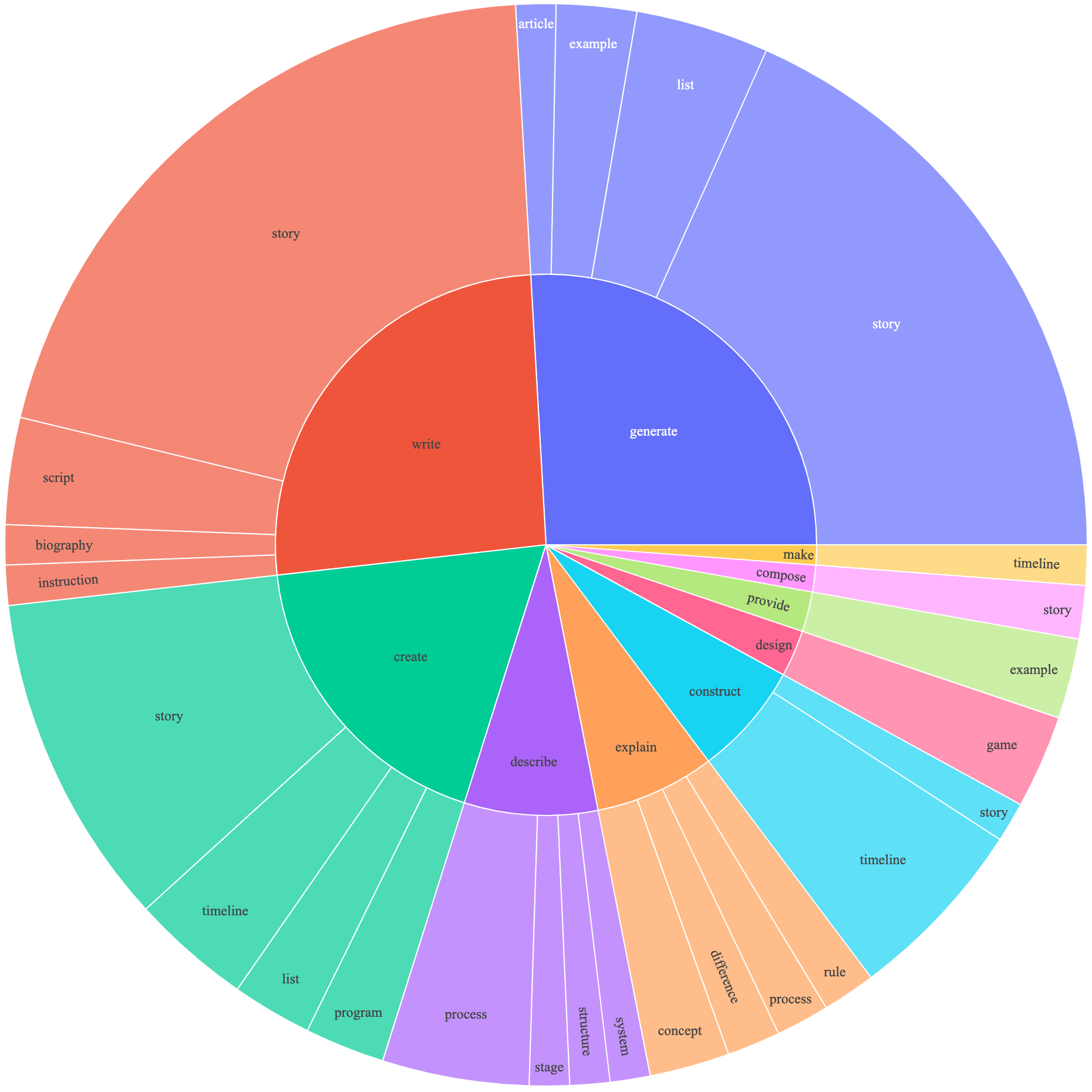}
    \caption{Top 800 data points selected by 7B model}
    \label{fig:7Bdata}
  \end{subfigure}
  \hfill
  \begin{subfigure}[b]{0.49\linewidth}
    \includegraphics[width=\linewidth]{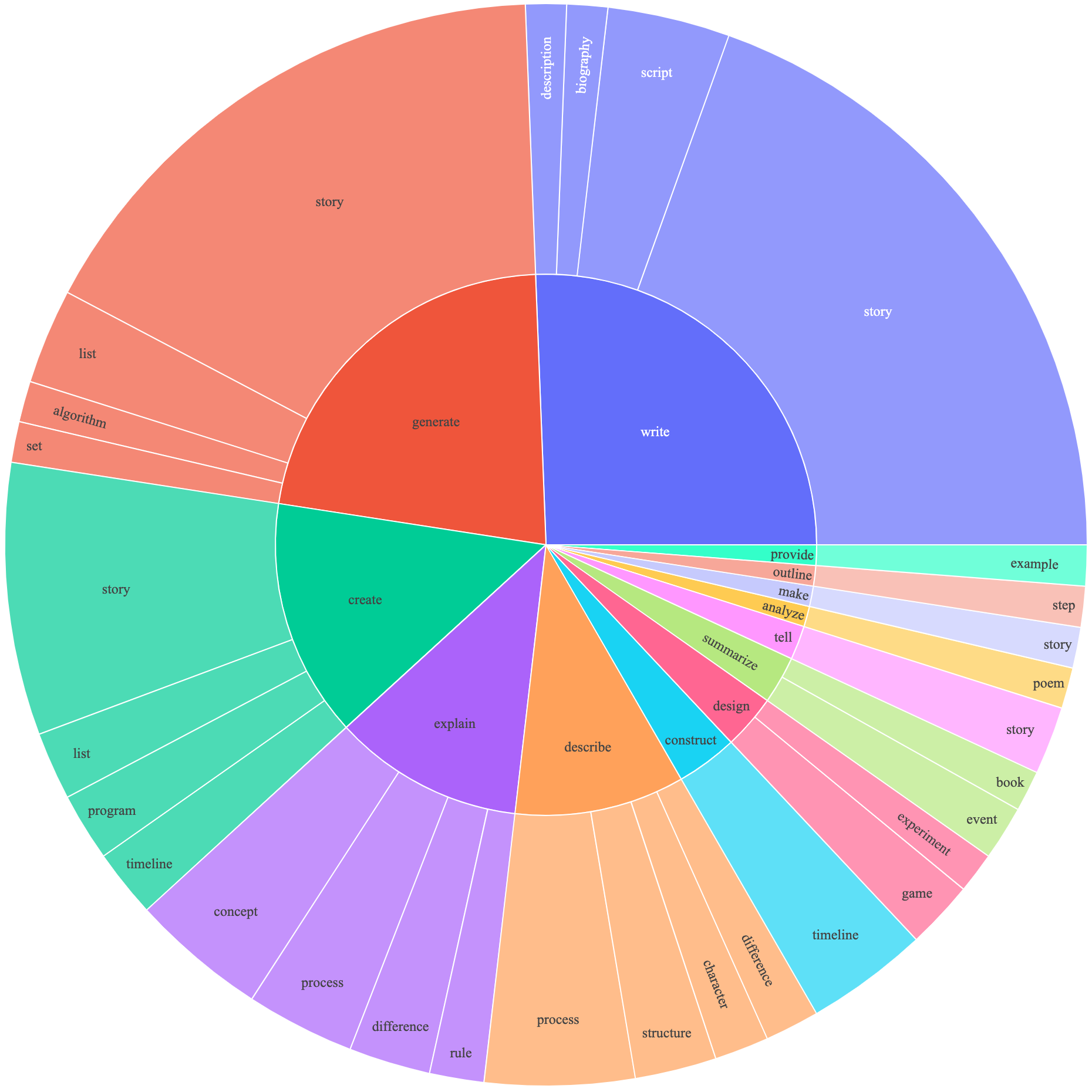}
    \caption{Top 800 data points selected by 13B model}
    \label{fig:13Bdata}
  \end{subfigure}
  \caption{Comparison of top 800 data points selected by different models}
  \label{fig:comparison}
\end{figure*}

\paragraph{Prompt Category Classification with LLM.}
Constituency parsing via Benepar, while helpful, does not effectively convey the semantics of the training data. Instead, we employed GPT-4 as classifier for a more precise semantically-oriented task classification.

In particular, we first classified Alpaca's seed instructions into 7 primary categories. The categories are then used to further classify the first 800 data entries selected by models 7B and 13B. As shown in Table.~\ref{tab.cata}, data in the "Problem Solving and Math" category had the most significant change between the two base models, increasing by 76.9\% from 7B to 13B. We hypothesize that this could be due to the substantial difference in mathematical and reasoning capabilities between the 7B and 13B models, thereby increasing the learnability of these SFT data for the 13B model.

\begin{table}[hbt]
\centering
\renewcommand\arraystretch{1.2}
\resizebox{\linewidth}{!}{%
\begin{tabular}{llll}
\noalign{\hrule height 1pt}
Category & 7B &  13B & $\Delta$ ($\Delta$\%) \\
\noalign{\hrule height 1pt}
Programming and Coding & 60 & 56 & -4(-6.6\%) \\
Planning and Organization & 63 & 57 & -6 (-9.5\%)  \\
Knowledge and Information Extraction&275&296&+21 (+7.6\%) \\
Language and Text Processing&53&45&-8 (-15.1\%) \\
Creative Writing and Entertainment&311&286&-25 (-8.0\%) \\
Problem Solving and Math&26&46&+20 (+76.9\%) \\
Recommendations and Suggestions&9&8&-1 (-11.1\%) \\
Others&3&6&-- \\
\noalign{\hrule height 1pt}
\end{tabular}
}
\caption{{\bf Comparison of first 800 data selected by LLaMA-7B and LLaMA-13B model based on category.}}
\label{tab.cata}
\end{table}

Finally, in Table.~\ref{tab:example_data}, we provide examples of data in Alpaca-4 with the highest normalized scores as computed using LLaMA-7B base model.

\begin{table*}[!t]
    \resizebox{\textwidth}{!}{%
    \begin{tabular}{@{}ccp{4cm}p{3cm}p{6cm}@{}}
    \toprule
    Category      & Score & Instruction                                                            & Input                                         & Output                                                                                                                                                                                    \\ \midrule
    Knowledge     & 12.4  & Identify the odd one out: Twitter, Instagram, Telegram.                &                                               & Telegram.                                                                                                                                                                                 \\
    Knowledge     & 9.80  & Convert the following sentence into the present continuous tense.      & He reads books.                               & He is reading books.                                                                                                                                                                      \\
    Knowledge     & 8.78  & Identify the most suitable adverb for the following sentence.          & She works.                                    & She works diligently.                                                                                                                                                                     \\
    Knowledge     & 8.21  & Analyze the tone of the following sentences.                           & I am so upset with you right now.             & The tone of the sentences is angry and accusatory.                                                                                                                                        \\
    Transcription & 7.38  & Rewrite the following sentence using active voice.                     & The news report was read by the captain.      & The captain read the news report.                                                                                                                                                         \\
    Knowledge     & 7.59  & How did Julius Caesar die?                                             &                                               & Julius Caesar was assassinated by a group of up to 60 conspirators, led by Gaius Cassius Longinus and Marcus Junius Brutus, in the Senate House on the Ides of March (15 March) of 44 BC. \\
    Mathematics   & 6.09  & Calculate the total surface area of a cube with a side length of 5 cm. &                                               & Total surface area of the cube: 150 cm$^2$.                                                                                                                                               \\
    Knowledge     & 5.96  & Provide one example of a cultural practice.                            &                                               & An example of a cultural practice is the Japanese tea ceremony.                                                                                                                           \\
    Mathematics   & 4.96  & Use the given data to calculate the median.                            & {[}2, 3, 7, 8, 10{]}.                         & The median of the given data is 7.                                                                                                                                                        \\
    Knowledge     & 5.15  & What is the product of 6 and 2?                                        &                                               & The product of 6 and 2 is 12.                                                                                                                                                             \\
    Knowledge     & 5.11  & Give two examples of a liquid.                                         &                                               & Two examples of liquids are water and oil.                                                                                                                                                \\
    Correction    & 5.12  & Evaluate this sentence for spelling and grammar mistakes.              & He finnished his meal and left the resturant. & He finished his meal and left the restaurant.                                                                                                                                             \\
    Knowledge     & 4.95  & Arrange the given numbers in ascending order.                          & 2, 4, 0, 8, 3.                                & 0, 2, 3, 4, 8.                                                                                                                                                                            \\
    Knowledge     & 4.66  & Given a set of numbers, find the maximum value.                        & Set: \{10, 3, 25, 6, 16\}.                    & The maximum value is 25.                                                                                                                                                                  \\
    Mathematics   & 4.53  & Convert from Celsius to Fahrenheit.                                    & Temperature in Celsius: 15.                   & Temperature in Fahrenheit: 59.                                                                                                                                                            \\
    Knowledge     & 4.16  & Transform the following sentence into the passive voice.               & I bought a book.                              & A book was bought by me.                                                                                                                                                                  \\
    Mathematics   & 4.12  & Convert the given equation into an algebraic expression.               & 3x + 5y = 9.                                  & 3x + 5y - 9 = 0.                                                                                                                                                                          \\
    Knowledge     & 4.10  & What are the three primary colors?                                     &                                               & The three primary colors are red, blue, and yellow.                                                                                                                                       \\
    None           & 3.93  & Render a 3D model of a house.                                          &                                               & \textless{}nooutput\textgreater This type of instruction cannot be fulfilled by a GPT model.                                                                                              \\ \bottomrule
    \end{tabular}
    }
    \caption{Example data with highest DavIR scores selected with LLaMA-7B base model.}
    \label{tab:example_data}
\end{table*}

\section{Statistical Analysis of AlpacaEval}\label{sect:statistical_analysis}
We first note that both the data selection procedure of DavIR and the model inference and evaluation procedures are deterministic (greedy decoding). Therefore, to establish a statistical significance analysis of the comparison between DavIR's performance to other methods, we performed bootstrap estimation with 1000 samples from the 805 questions of the AlpacaEval dataset, giving us the 95\% confidence interval shown in Table.~\ref{tab:gemma2b_with_CI}.

As the bootstrap estimations of the performances on AlpacaEval is highly Gaussian for all experiments with similar variance, we performed t-test between DavIR and all other baseline method. The p-values , shown in Table.~\ref{tab:t_test} below show that the performance gain of DavIR over other methods is indeed significant across number of samples (with p-values very close to 0).

\begin{table*}[t]
\resizebox{\linewidth}{!}{%
\begin{tabular}{@{}llllllllll@{}}
\toprule
\textbf{Method}             & \textbf{Selection} & \textbf{3000}                 & \textbf{4000}                 & \textbf{5000}                 & \textbf{6000}                 & \textbf{7000}                 & \textbf{8000}                 & \textbf{9000}                 & \textbf{10000}                \\ \midrule
Random                      & N/A                & 10.7(8.6$\sim$12.8)           & 12.7(10.7$\sim$14.7)          & 15.9(13.6$\sim$18.3)          & 16.1(13.8$\sim$18.5)          & 17.0(14.8$\sim$19.3)          & 16.4(14.1$\sim$18.8)          & 16.7(14.5$\sim$18.9)          & 17.6(15.1$\sim$20.2)          \\
\multirow{2}{*}{EL2N}       & Lowest             & 10.4(8.5$\sim$12.2)           & 13.5(11.3$\sim$15.7)          & \textbf{17.6(15.4$\sim$19.8)} & 17.6(15.4$\sim$20.1)          & 19.0(16.5$\sim$21.6)          & 18.5(16.2$\sim$20.9)          & 18.7(16.5$\sim$21.3)          & 18.4(16.0$\sim$21.0)          \\
                            & Highest            & 10.0(8.2$\sim$11.9)           & 11.1(9.3$\sim$13.2)           & 11.2(9.3$\sim$13.2)           & 11.8(9.7$\sim$13.9)           & 12.4(10.3$\sim$14.7)          & 13.2(11.2$\sim$15.4)          & 13.9(11.8$\sim$16.2)          & 14.3(12.3$\sim$16.7)          \\
\multirow{2}{*}{Forgetting} & Lowest             & 10.4(8.6$\sim$12.3)           & 12.8(10.7$\sim$15.0)          & 16.6(14.4$\sim$19.1)          & 16.5(14.2$\sim$18.9)          & 17.1(14.8$\sim$19.5)          & 17.0(14.7$\sim$19.4)          & 17.7(15.6$\sim$20.1)          & 16.7(14.6$\sim$19.1)          \\
                            & Highest            & 9.5(7.7$\sim$11.4)            & 12.4(10.5$\sim$14.6)          & 13.4(11.3$\sim$15.5)          & 15.5(13.2$\sim$17.9)          & 16.7(14.6$\sim$18.9)          & 17.8(15.6$\sim$20.3)          & 18.9(16.6$\sim$21.3)          & 18.2(15.7$\sim$20.5)          \\
\multirow{2}{*}{DataInf}    & Lowest             & 10.2(8.4$\sim$12.1)           & 12.3(10.3$\sim$14.4)          & 13.9(11.9$\sim$16.0)          & 15.1(12.7$\sim$17.4)          & 15.0(13.0$\sim$17.4)          & 14.3(12.1$\sim$16.6)          & 15.3(13.2$\sim$17.5)          & 15.4(13.3$\sim$17.7)          \\
                            & Highest            & 10.3(8.4$\sim$12.2)           & \textbf{13.8(11.6$\sim$15.9)} & 15.9(13.7$\sim$18.2)          & 18.6(16.2$\sim$21.2)          & 18.8(16.2$\sim$21.2)          & 19.8(17.4$\sim$22.6)          & 18.5(16.1$\sim$20.9)          & 18.9(16.4$\sim$21.3)          \\
\multirow{2}{*}{RHO}        & Lowest             & 9.6(7.9$\sim$11.6)            & 12.1(10.0$\sim$14.3)          & 13.8(11.7$\sim$16.0)          & 15.3(13.1$\sim$17.5)          & 15.7(13.5$\sim$18.0)          & 16.1(13.8$\sim$18.4)          & 16.2(13.9$\sim$18.4)          & 16.6(14.2$\sim$19.0)          \\
                            & Highest            & 9.9(8.1$\sim$11.9)            & 12.2(10.3$\sim$14.5)          & 14.6(12.3$\sim$16.8)          & 15.8(13.5$\sim$18.0)          & 15.4(13.1$\sim$17.6)          & 15.5(13.2$\sim$17.9)          & 16.5(14.2$\sim$18.9)          & 16.7(14.4$\sim$19.0)          \\
\multirow{2}{*}{DavIR}      & Lowest             & 10.5(8.7$\sim$12.5)           & 12.7(10.7$\sim$14.8)          & 14.2(12.0$\sim$16.4)          & 14.3(12.1$\sim$16.5)          & 13.1(11.1$\sim$15.3)          & 12.0(9.9$\sim$13.8)           & 13.2(11.1$\sim$15.4)          & 13.6(11.4$\sim$15.7)          \\
                            & Highest            & \textbf{10.8(8.8$\sim$12.9)}  & 13.1(11.0$\sim$15.2)          & 17.1(14.8$\sim$19.6)          & \textbf{20.2(17.7$\sim$22.7)} & \textbf{20.2(17.6$\sim$22.7)} & \textbf{20.3(17.9$\sim$22.9)} & \textbf{20.2(17.8$\sim$22.5)} & \textbf{19.7(17.3$\sim$22.1)} \\
Full                        & N/A                & \textit{18.3(15.9$\sim$20.6)} & \textit{18.3(15.9$\sim$20.6)} & \textit{18.3(15.9$\sim$20.6)} & \textit{18.3(15.9$\sim$20.6)} & \textit{18.3(15.9$\sim$20.6)} & \textit{18.3(15.9$\sim$20.6)} & \textit{18.3(15.9$\sim$20.6)} & \textit{18.3(15.9$\sim$20.6)} \\ \bottomrule
\end{tabular}
}
\caption{{\bf DavIR comparison with baselines with 95\% Confidence Interval.}}
\label{tab:gemma2b_with_CI}
\end{table*}

\begin{table*}[t]
\resizebox{\linewidth}{!}{%
    \begin{tabular}{@{}llllllllll@{}}
    \toprule
    \textbf{Method}             & \textbf{Selection} & \textbf{3000} & \textbf{4000} & \textbf{5000} & \textbf{6000} & \textbf{7000} & \textbf{8000} & \textbf{9000} & \textbf{10000} \\ \midrule
    Random                      & N/A                & 5.3E-03       & 2.4E-16       & 1.3E-93       & 0.0E+00       & 0.0E+00       & 0.0E+00       & 0.0E+00       & 8.8E-229       \\
    \multirow{2}{*}{EL2N}       & Lowest             & 4.4E-26       & -             & -             & 3.1E-307      & 1.3E-78       & 1.2E-187      & 2.8E-134      & 5.7E-99        \\
                                & Highest            & 6.0E-65       & 6.4E-276      & 0.0E+00       & 0.0E+00       & 0.0E+00       & 0.0E+00       & 0.0E+00       & 0.0E+00        \\
    \multirow{2}{*}{Forgetting} & Lowest             & 3.9E-22       & 6.2E-08       & 5.2E-17       & 0.0E+00       & 0.0E+00       & 0.0E+00       & 1.5E-321      & 0.0E+00        \\
                                & Highest            & 1.7E-160      & 1.8E-43       & 0.0E+00       & 0.0E+00       & 0.0E+00       & 0.0E+00       & 3.9E-108      & 1.1E-146       \\
    \multirow{2}{*}{DataInf}    & Lowest             & 2.9E-38       & 2.4E-63       & 0.0E+00       & 0.0E+00       & 0.0E+00       & 0.0E+00       & 0.0E+00       & 0.0E+00        \\
                                & Highest            & 2.6E-29       & -             & -             & 4.4E-136      & 4.4E-116      & 1.1E-17       & 7.9E-180      & 3.8E-46        \\
    \multirow{2}{*}{RHO}        & Lowest             & 3.5E-141      & 1.8E-84       & 0.0E+00       & 0.0E+00       & 0.0E+00       & 0.0E+00       & 0.0E+00       & 0.0E+00        \\
                                & Highest            & 6.6E-86       & 5.2E-67       & 0.0E+00       & 0.0E+00       & 0.0E+00       & 0.0E+00       & 0.0E+00       & 0.0E+00        \\
    DavIR                       & Lowest             & 1.1E-12       & 1.5E-16       & 0.0E+00       & 0.0E+00       & 0.0E+00       & 0.0E+00       & 0.0E+00       & 0.0E+00        \\
    Full                        & N/A                & -             & -             & -             & 3.9E-189      & 6.1E-193      & 1.8E-224      & 3.2E-209      & 6.2E-124       \\ \bottomrule
    \end{tabular}
}
\caption{{\bf p-values of t-test comparing DavIR and all other selection methods presented in Table.~\ref{tab:davir_vs_everyone} and Table.~\ref{tab:gemma2b_with_CI}.} Note that since the hypothesis s that DavIR out-performs other methods, only results where DavIR out-performs the baseline methods have corresponding p-values.}
\label{tab:t_test}
\end{table*}